\documentclass[10pt,twocolumn,letterpaper]{article}

\usepackage{iccv}
\usepackage{times}
\usepackage{epsfig}
\usepackage{graphicx}
\usepackage{amsmath}
\usepackage{amssymb}

\usepackage{algorithm}
\usepackage{algorithmic}
\usepackage{multirow}
\usepackage{caption}
\usepackage{paralist}
\usepackage{amsthm}
\usepackage{listings}

\def \x {\mathbf{x}}

\def \OO {\mathcal{O}}

\def \w {\mathbf{w}}

\def \LL {\mathcal{L}}

\newtheorem{thm}{Theorem}
\newtheorem{prop}{Proposition}

\usepackage[accsupp]{axessibility}  

\usepackage[breaklinks=true,bookmarks=false]{hyperref}

\iccvfinalcopy 

\def\iccvPaperID{11010} 

\ificcvfinal\fi

\begin{document}

\title{Stable Cluster Discrimination for Deep Clustering}

\author{Qi Qian\\
Alibaba Group, Bellevue, WA 98004, USA\\
{\tt\small qi.qian@alibaba-inc.com}
}

\maketitle
\ificcvfinal\thispagestyle{empty}\fi

\begin{abstract}
   Deep clustering can optimize representations of instances (i.e., representation learning) and explore the inherent data distribution (i.e., clustering) simultaneously, which demonstrates a superior performance over conventional clustering methods with given features. However, the coupled objective implies a trivial solution that all instances collapse to the uniform features. To tackle the challenge, a two-stage training strategy is developed for decoupling, where it introduces an additional pre-training stage for representation learning and then fine-tunes the obtained model for clustering. Meanwhile, one-stage methods are developed mainly for representation learning rather than clustering, where various constraints for cluster assignments are designed to avoid collapsing explicitly. Despite the success of these methods, an appropriate learning objective tailored for deep clustering has not been investigated sufficiently. In this work, we first show that the prevalent discrimination task in supervised learning is unstable for one-stage clustering due to the lack of ground-truth labels and positive instances for certain clusters in each mini-batch. To mitigate the issue, a novel stable cluster discrimination (SeCu) task is proposed and a new hardness-aware clustering criterion can be obtained accordingly. Moreover, a global entropy constraint for cluster assignments is studied with efficient optimization. Extensive experiments are conducted on benchmark data sets and ImageNet. SeCu achieves state-of-the-art performance on all of them, which demonstrates the effectiveness of one-stage deep clustering. Code is available at \url{https://github.com/idstcv/SeCu}.
\end{abstract}

\section{Introduction}

Clustering is a fundamental task in unsupervised learning. Given features of instances, the unlabeled data set will be partitioned into multiple clusters, where instances from the same cluster are similar according to the measurement defined by a distance function. With fixed features, most of research efforts focus on studying appropriate distance functions and ingenious algorithms have been proposed by different measurements, e.g., k-means clustering~\cite{Lloyd82}, spectral clustering~\cite{Luxburg07}, subspace clustering~\cite{ElhamifarV13}, etc.

\begin{figure}[t]
\centering
\begin{minipage}{0.47\linewidth}
\centering
\includegraphics[height = 1.1in]{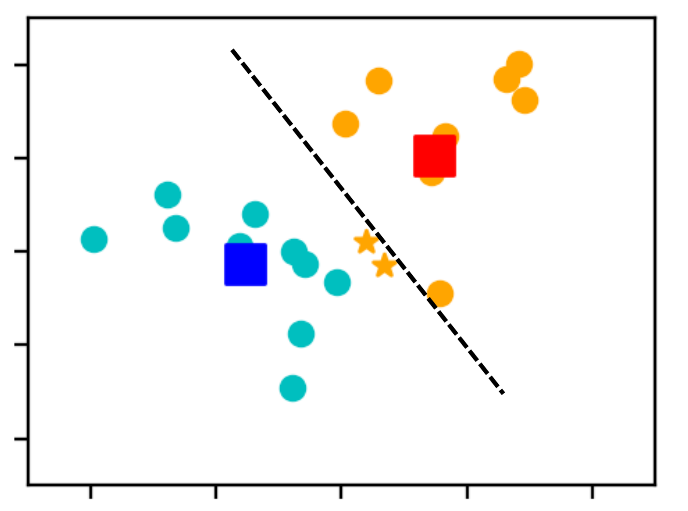}
\mbox{\footnotesize k-means in CoKe: acc=0.9}
\end{minipage}
\begin{minipage}{0.47\linewidth}
\centering
\includegraphics[height = 1.1in]{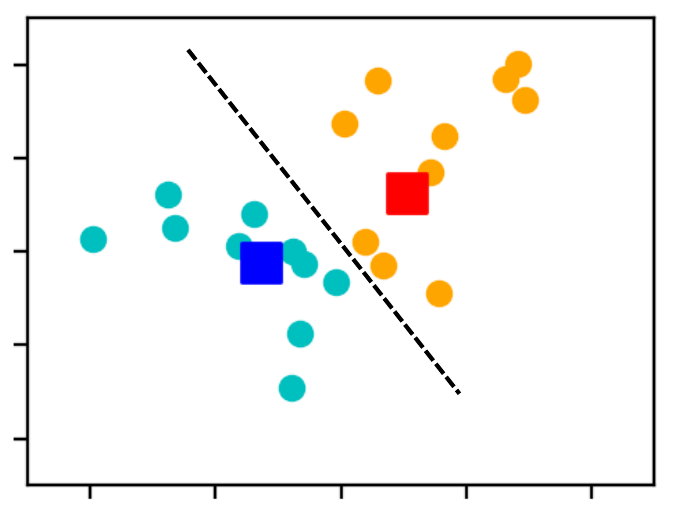}
\mbox{\footnotesize our proposal SeCu: acc=1.0}
\end{minipage}
\caption{An illustration of the proposed method. 10 data points are randomly sampled from two different Gaussian distributions, respectively. Points with the same color are from the same distribution, while squares denote the corresponding cluster centers obtained by different methods. Star indicates the mis-classified data. Unlike k-means that assigns the uniform weight for different instances, our method considers the hardness of instances and is better for discrimination based clustering.}\label{fig:illu}
\end{figure}

With the development of deep neural networks, deep learning is capable of learning representations from raw materials and demonstrates the dominating performance on various supervised tasks~\cite{KrizhevskySH12}. Thereafter, it is also introduced to clustering recently~\cite{CaronBJD18,DangD0WH21, GansbekeVGPG20,HuangGZ20,JiVH19,Li0LPZ021,XieGF16,YangFSH17,ZhongW0HDNL021}. Unlike the conventional clustering, representations of instances are learned with cluster assignments and centers simultaneously in deep clustering, which is more flexible to capture the data distribution. However, the coupled objective can result in a trivial solution that all instances collapse to the uniform features~\cite{CaronBJD18} and designing appropriate clustering criterion in the new scenario becomes challenging. 

Many deep clustering algorithms~\cite{DangD0WH21,GansbekeVGPG20,ZhongW0HDNL021} exploit a two-stage training strategy to decouple representation learning and clustering to avoid collapsing. The recent progress in unsupervised representation learning~\cite{abs-2106-08254,CaronBJD18,ChenK0H20,abs-2111-06377,He0WXG20,coke} shows that informative features capturing the semantic similarity between instances without collapsing can be obtained by pre-training a large set of unlabeled data. Inspired by the observation, those methods leverage a pre-training stage and then focus on developing algorithms to optimize clustering by fine-tuning a pre-trained model in the second stage. By refining the relations between nearest neighbors obtained from the pre-training stage, two-stage methods achieve a significantly better performance than one-stage ones without sufficient representation learning~\cite{HuangGZ20,JiVH19}.

However, the objective of pre-training can be inconsistent with that of clustering, which results in a sub-optimal performance for deep clustering. Note that different from supervised learning where the objective is explicitly defined by labels, that for unsupervised representation learning is arbitrary and various pretext tasks have been proposed, e.g., instance discrimination~\cite{ChenK0H20,He0WXG20}, cluster discrimination~\cite{CaronBJD18,coke}, masked modeling~\cite{abs-2106-08254,abs-2111-06377}, etc. Most of existing two-stage clustering methods adopt instance discrimination, i.e., SimCLR~\cite{ChenK0H20}, for pre-training, whereas it aims to identify each instance as an individual class and the objective is different from clustering that aims to group instances from the same cluster. Moreover, SwAV~\cite{CaronMMGBJ20} demonstrates that clustering itself is also an effective pretext task for representation learning. Therefore, we focus on facilitating one-stage deep clustering that optimizes representations and clustering simultaneously in this work.

Most of existing one-stage methods are proposed solely for representation learning. To tackle the collapsing issue, research efforts are mainly devoted to developing appropriate constraints for cluster assignments, especially for online deep clustering. For example, \cite{AsanoRV20a,CaronMMGBJ20} apply the balanced constraint that each cluster has the same number of instances for clustering. \cite{coke} further relaxes the balanced constraint to a lower-bound size constraint that limits the minimal size of clusters and demonstrates a more flexible assignment. 

After obtaining cluster assignments as pseudo labels, representations and cluster centers can be optimized as in the supervised scenario. For example, \cite{CaronMMGBJ20,CaronTMJMBJ21} learn the encoder network and cluster centers by solving a classification problem with the standard cross entropy loss. \cite{coke} has the same classification task for representation learning while adopting k-means for updating cluster centers. Although these methods achieve a satisfied performance on representation learning, a learning objective tailored for deep clustering has not attracted sufficient attentions.

In this work, we investigate the effective learning task for one-stage deep clustering. By analyzing the standard cross entropy loss for supervised learning, we find that it can be unstable for unsupervised learning. Concretely, the gradient for updating cluster centers consists of two ingredients: gradient from positive instances of the target cluster and that from irrelevant negative ones. However, with a limited mini-batch size in stochastic gradient descent (SGD), there can be no positive ones for a large proportion of clusters (e.g., 90\% on ImageNet) at each iteration. Due to the lack of positive instances, the influence from negative ones is dominating. Unlike supervised learning where labels are fixed, the cluster assignments can change during training in deep clustering. Therefore, the noise from the large variance of negative instances will be accumulated, which makes the optimization unstable. 

To mitigate the problem, we propose a stable cluster discrimination (SeCu) task for deep clustering, which stops the gradient from negative instances for updating cluster centers in the cross entropy loss. Compared with k-means in \cite{coke}, where positive instances have the uniform weight for updating centers, SeCu as a discrimination task considers the hardness of instances and a large weight will be assigned to hard instances for updating. Fig.~\ref{fig:illu} illustrates the hardness-aware clustering criterion implied by SeCu. Besides, we improve the cluster assignment by developing an entropy constraint that regularizes the entropy of assignments over the entire data set. Compared to the optimization with the size constraint~\cite{coke}, our method can reduce the number of variables and hyper-parameters, and thus make the learning more convenient. The main contributions of this work can be summarized as follows.
\begin{itemize}
    \item A novel task is proposed for one-stage deep clustering. SeCu tailors the supervised cross entropy loss by eliminating the influence from the negative instances in learning cluster centers, which makes the training stable when ground-truth labels are unavailable.
    \item A global entropy constraint is exploited to balance the size of clusters and an efficient closed-form solution is developed for online assignment with the constraint. 
    \item A simple framework with a single loss function and encoder is introduced for deep clustering in Eqn.~\ref{eq:obj}. The proposed method is evaluated with the standard protocol on benchmark data sets and ImageNet~\cite{RussakovskyDSKS15}. The superior performance of SeCu confirms its effectiveness for deep clustering.
\end{itemize}

\section{Related Work}
In this section, we briefly review unsupervised representation learning and then focus on clustering.

\subsection{Unsupervised Representation Learning}
By leveraging massive unlabeled data, unsupervised representation learning can learn a pre-trained model that encoders the semantic information from data and helps downstream tasks with fine-tuning. Compared with supervised learning containing labels, one major challenge in unsupervised learning is defining appropriate positive pairs for representation learning. SimCLR~\cite{ChenK0H20} and MoCo~\cite{He0WXG20} apply instance discrimination as the objective that generates positive pairs from diverse views of the same instance and considers other instances as the negative ones. Then, BYOL~\cite{GrillSATRBDPGAP20} and SimSiam~\cite{ChenH21} demonstrate that with appropriate neural network architectures, the applicable models can be obtained with only positive pairs. Besides the work optimizing instance space, Barlow Twins~\cite{ZbontarJMLD21} defines positive pairs in feature space and also learns effective pre-trained models. Recently, learning with masked modeling~\cite{abs-2106-08254,abs-2111-06377} achieves success on the specific backbone of vision transformer~\cite{DosovitskiyB0WZ21} and shows the better fine-tuning performance than the generic objective on the ResNet~\cite{HeZRS16}. 

\subsection{Deep Clustering}
Compared with instance discrimination, cluster discrimination aims to obtain positive pairs consisting of different instances for representation learning. Therefore, it has to learn representations and relations between instances simultaneously. DeepCluster~\cite{CaronBJD18} obtains membership of instances with an offline k-means and then optimizes representations alternately. Some work focuses on representation learning and introduces the balanced constraint that each cluster has the same number of instances for cluster assignments~\cite{AsanoRV20a,CaronMMGBJ20} to mitigate the collapsing problem. Besides, DINO~\cite{CaronTMJMBJ21} applies an additional momentum encoder to further improve the representation. However, the constraint assumes a well balanced distribution, which is hard to capture the ground-truth data distribution. Thereafter, CoKe~\cite{coke} relaxes the constraint to only lower-bound the minimal size of clusters to model the distribution more flexibly.

Besides representation learning, there are many works developed solely for better clustering~\cite{DangD0WH21,GansbekeVGPG20,JiVH19,Li0LPZ021,XieGF16,YangFSH17,ZhongW0HDNL021}. To handle the collapsing issue, two-stage methods~\cite{DangD0WH21,GansbekeVGPG20,ZhongW0HDNL021} leverage the representations from a pre-trained model to obtain nearest neighbors and fine-tune the model for clustering accordingly. In addition, the objective of pre-training can be included for clustering as multi-task learning~\cite{Li0LPZ021}. The work closest to ours is CoKe~\cite{coke} that optimizes representations and clustering simultaneously. Compared with the conventional k-means in CoKe, we propose a novel stable cluster discrimination objective tailored for one-stage deep clustering. Moreover, an entropy constraint is introduced to reduce the number of parameters while demonstrating the competitive performance.

\section{Stable Cluster Discrimination}
\subsection{Cluster Discrimination}

Given a data set with $N$ images $\{x_i\}_{i=1}^N$, if the corresponding labels are accessible as $\{y_i\}_{i=1}^N$, the literature from distance metric learning~\cite{QianSSHTLJ19} shows that appropriate representations of examples can be learned by optimizing a classification task as
\[\min_{\theta} \sum_{i=1}^N \ell(x_i,y_i;\theta)\]
where $\ell(\cdot)$ is the cross entropy loss with the normalized Softmax operator as suggested in \cite{QianSSHTLJ19}. $\theta$ denotes the parameters of the deep neural network.

For unsupervised learning that the label information is unavailable, the objective of cluster discrimination with $K$ clusters can be written as
\begin{eqnarray}\label{eq:unlabel}
\min_{\theta_f,\{\w_j\}, y_i\in\Delta} \mathcal{L} = \sum_{i=1}^N\sum_{j=1}^K -y_{i,j}\log(p_{i,j})
\end{eqnarray}
where $y_i$ denotes the learnable label for $x_i$ and $\Delta = \{y_i|\sum_{j=1}^K y_{i,j}=1, \forall j, y_{i,j}\in \{0,1\}\}$. With the normalized Softmax operator, the prediction $p_{i,j}$ can be computed 
\begin{eqnarray}\label{eq:prob}
p_{i,j} = \frac{\exp(\x_i^\top \w_j/\lambda)}{\sum_{k=1}^K \exp(\x_i^\top \w_k/\lambda)}
\end{eqnarray}
where $\x_i = f(x_i)$ and $f(\cdot)$ denotes the encoder network. $\theta_f$ contains parameters of $f$. $\{\w_j\}_{j=1}^K$ consists of cluster centers. $\lambda$ is the parameter of temperature and $\forall i,j, \|\x_i\|_2 = \|\w_j\|_2=1$ by normalization.

Compared with the supervised counterpart, the problem in Eqn.~\ref{eq:unlabel} has to optimize cluster assignments $\{y\}$, cluster centers $\{\w\}$, and representation encoder network $f$ simultaneously. While most of existing work focus on optimizing $\{y\}$ with different constraints to avoid collapsing, this work investigates the learning objective for $\{\w\}$ and $f$, and a new clustering criterion is introduced accordingly.

\subsection{Stable Loss for Optimization with Mini-batch}
Unlike supervised learning, where labels are fixed for examples, cluster assignments $\{y\}$ in unsupervised learning are dynamic with the training of instance representations and cluster centers. Therefore, the original cross entropy loss becomes unstable for unsupervised discrimination. The problem can be elaborated by analyzing the updating criterion for cluster centers.

Letting $y_i$ denote the label of $\x_i$, the gradient of $\w_j$ from the standard cross entropy loss in Eqn.~\ref{eq:unlabel} can be computed as
\[\nabla_{\w_j}\mathcal{L} = \frac{1}{\lambda} (\sum_{i:y_i=j} (p_{i,j}-1)\x_i + \sum_{k:y_k\neq j} p_{k,j}\x_k)\]
The former term is to pull the center $\w_j$ to the assigned instances and the latter one is to push it away from instances of other clusters. However, the following Propositions show that this updating is unstable for deep clustering. The detailed proof can be found in the appendix.

\begin{prop}\label{prop:1}
When sampling a mini-batch of $b$ instances for $K$ clusters, there are no positive instances for at least $K-b$ clusters.
\end{prop}

\begin{prop} \label{prop:2}
Let $\mathrm{Var}_{pos}$ and $\mathrm{Var}_{neg}$ be the variance of sampled positive and negative instances, respectively. Assuming that each instance has unit norm and the norm of the cluster mean is $a$, we have
$\mathrm{Var}_{neg} = \OO(\frac{1}{1-a^2}) \mathrm{Var}_{pos}$.
\end{prop}

\paragraph{Remark} Proposition~1 implies that with cross entropy loss, a large proportion of cluster centers will be updated only by negative instances. Proposition~2 indicates that the variance of sampled negative instances is much larger than that of positive ones when each cluster is compact as $a$ approaching 1. Due to the small size of a mini-batch in training deep neural networks, the variance cannot be reduced sufficiently. 

When $K=10,000$ and $b=1,024$ as in our multi-clustering settings for ImageNet, at least about 90\% centers will only access a mini-batch of negative instances at each iteration. Without ground-truth labels, the bias in updating will be accumulated and mislead the learning process.

To mitigate the problem, we propose to eliminate the direction of negative instances from gradient for stable training as
\[\nabla_{\w_j}\mathcal{L} = \frac{1}{\lambda} \sum_{i:y_i=j} (p_{i,j}-1)\x_i\]
The corresponding stable cluster discrimination loss becomes
\begin{align}\label{eq:pce}
&\ell_{\mathrm{SeCu}}(x_i,y_i) = \nonumber\\
&-\log(\frac{\exp(\x_i^\top \w_{y_i}/\lambda)}{\exp(\x_i^\top \w_{y_i}/\lambda)+\sum\limits_{k:k\neq y_i}\exp(\x_i^\top\tilde{\w}_k/\lambda)})
\end{align}
where $\tilde{\w}_k$ denotes $\w_k$ with the stop-gradient operator. Compared with the standard cross entropy loss, cluster centers in SeCu loss will be updated only by positive instances, which is more stable for deep clustering when optimizing with mini-batches. 

In addition, the analysis in the following theorem demonstrates the novel hardness-aware clustering criterion for deep clustering implied by the proposed objective.

\begin{thm}
When fixing $\{y_i\}$ and $\{\x_i\}$, let $\{\w^*\}$ be the optimal solution of the problem with loss function in Eqn.~\ref{eq:pce} and assume $\forall i, \|\x_i\|_2=1; \forall j, \|\w_j\|_2=1$, then we have
\begin{eqnarray}\label{eq:w}
\w_j^* = \Pi_{\|\w\|_2=1}(\frac{\sum_{i:y_i=j} (1-p_{i,j})\x_i}{\sum_{i:y_i=j} 1- p_{i,j}})
\end{eqnarray}
where $\Pi_{\|\w\|_2=1}$ projects the vector to the unit norm.
\end{thm}

\paragraph{Remark} For k-means in CoKe~\cite{coke}, cluster centers will be averaged over all assigned instances with the uniform weight
\begin{eqnarray}\label{eq:cokew}
\w_j = \Pi_{\|\w\|_2=1}(\frac{\sum_{i:y_i=j} \x_i}{\sum_{i} 1(y_i=j)})
\end{eqnarray}
where $1(\cdot)$ is the indicator function.
On the contrary, our proposal considers the weight of each instance by its hardness, i.e., $p_{i,y_i}$. By assigning a large weight to hard instances, the corresponding center can capture the distribution of hard instances better, which is illustrated in Fig.~\ref{fig:illu}. 

Our formulation also implies CoKe as a special case. Concretely, by increasing the temperature $\lambda$ for updating $\w$ to be infinite, we have $p_{i,j}=1/K$ and Eqn.~\ref{eq:w} will degenerate to Eqn.~\ref{eq:cokew}.

Finally, besides SGD, Eqn.~\ref{eq:w} also suggests an alternative updating strategy for $\w$. Since $p_{i,j}$ is computed by $\w$, it may still require multiple iterations with Eqn.~\ref{eq:w} for convergence. However, the representations of instances are also updated by training and we can keep a single update for centers at each iteration and improve centers along with the learning of representations. Compared with SGD for learning cluster centers, the strategy eliminates the learning rate for centers with a comparable performance as shown in the ablation study. 

Given the learning objective SeCu, we will elaborate the proposed deep clustering method in next subsection.

\subsection{SeCu for Deep Clustering}
With the proposed loss function, the objective of stable cluster discrimination for deep clustering can be written as
\begin{align}\label{eq:obj}
&\min_{\theta_f,\{\w_j\}, y_i\in\Delta} \sum_{i=1}^N \ell_{\mathrm{SeCu}}(x_i, y_i)\nonumber\\
&s.t.\quad  h_m(Y)\geq b_m,\quad m=1,\dots, M
\end{align}
where $Y=[y_1,\dots, y_N]$, $M$ denotes the number of constraints, and $h_m(\cdot)$ is the $m$-th constraint for cluster assignment. Considering that variables are entangled, the problem is solved alternately.

\paragraph{Update of $\theta_f$}
First, when fixing $\{y^{t-1}\}$ and $\{\w^{t-1}\}$ from the last epoch, the sub-problem for representation learning in the $t$-th epoch becomes
\[\min_{\theta_f} \sum_{i=1}^N \ell_{\mathrm{SeCu}}(x_i, y_i^{t-1})\]
where the SeCu loss degenerates to the standard cross entropy loss with fixed $\{\w\}$ and can be optimized with SGD. The one-hot label $y^{t-1}$ is kept from the $(t-1)$-th epoch, which makes the optimization consistent between adjacent epochs but the updating for representations may be delayed. To incorporate the information from the current epoch, we include two views of augmentations for the individual instance at  each iteration, which is prevalent for representation learning~\cite{ChenK0H20,He0WXG20,coke}. 

Let $x_i^1$ and $x_i^2$ be two perturbed views of the original image with random augmentations, and the prediction is
\[p_{i,j}^s = \frac{\exp(\x_i^{s\top}\w_{j}^{t-1}/\lambda)}{\sum_j^K\exp(\x_i^{s\top}\w_{j}^{t-1}/\lambda)};\quad s=\{1,2\}\]
where $\x_i^s = f_t(x_i^s)$ is extracted by the current encoder. The soft label for each view can be obtained as~\cite{coke}
\begin{align}\label{eq:soft}
y_i^1 = \tau y_i^{t-1} + (1-\tau) p_i^2;\quad y_i^2 = \tau y_i^{t-1} + (1-\tau) p_i^1
\end{align}
The soft label contains the prediction from the other view of the same instance, which optimizes the consistency between different views at the same iteration.
The problem with two-view optimization can be written as
\[\min_{\theta_f} \sum_{i=1}^N \ell_{\mathrm{SeCu}}(x_i^1, y_i^1)+\ell_{\mathrm{SeCu}}(x_i^2, y_i^2)\]

\paragraph{Update of $y$} 
When fixing $\x_i^t$ and $\{\w^t\}$, the cluster assignment can be updated by solving the problem
\begin{align*}
&\min_{y_i\in\Delta} -\sum_{i=1}^N\sum_{j=1}^K y_{i,j} \log(p_{i,j})\\
&s.t.\quad  h_m(Y)\geq b_m,\quad m=1,\dots, M
\end{align*}
where $p_{i,j}$ is defined by $\x_i^t$ and $\w^t$ as in Eqn.~\ref{eq:prob}.
Without the constraints $\{h_m\}$, the objective implies a greedy solution that assigns each instance to the most related cluster. It can incur the trivial solution of collapsing, especially for online deep clustering, where each instance can only be accessed once in each epoch and cluster assignments cannot be refined with multiple iterations over the entire data. Therefore, many effective constraints are developed to mitigate the challenge. 

First, we investigate the size constraint that is prevalent in existing methods~\cite{CaronMMGBJ20,coke}. Following \cite{coke}, the cluster size can be lower-bounded as
\begin{align}\label{eq:obsize}
&\min_{y_i\in\Delta} -\sum_{i=1}^N\sum_{j=1}^K y_{i,j} \log(p_{i,j})\nonumber\\
&s.t.\quad  \sum_i y_{i,j} \geq \gamma N/K,\quad j=1,\dots, K
\end{align}
where $\gamma$ is the proportion to the average size and $\gamma=1$ implies the balanced constraint. Let $\rho_j$ denote dual variables for the $j$-th lower-bound constraint. The problem becomes
\[\max_{\rho:\rho\geq 0}\min_{y_i\in\Delta} -\sum_i\sum_j y_{i,j} \log(p_{i,j}) - \sum_j \rho_j (\sum_i y_{i,j}-\gamma N/K)\]
When a mini-batch of instances arrive, the cluster assignment for each instance can be obtained via a closed-form solution
\begin{eqnarray}\label{eq:y}
y_{i,j}^t = \left\{\begin{array}{ll}1\quad&j=\arg\min_j -\log(p_{i,j}) - \rho_j  \\ 0\quad&o.w. \end{array} \right.
\end{eqnarray}
Then, the dual variables will be updated by stochastic gradient ascent. More details can be found in \cite{coke}.

The size constraint is capable of avoiding collapsing problem explicitly, but it introduces additional dual variables to help assignment and has to optimize them in training. To simplify the optimization, a global entropy constraint is investigated to balance the distribution over all clusters. Compared with the size constraint, the bound for the cluster size is implicit with the entropy constraint. Nevertheless, as we will demonstrate, it can help reduce the number of variables and hyper-parameters, which is more friendly for users. 

Given the set of cluster assignments $\{y\}$, the entropy of the whole data set can be defined as
\[H(y) = -\sum_{j=1}^K \frac{\sum_i^N y_{i,j}}{N}\log(\frac{\sum_i^N y_{i,j}}{N})\]
With the entropy as the regularization, the objective can be written as
\begin{align*}
&\min_{y_i\in\Delta} -\sum_{i=1}^N\sum_{j=1}^K y_{i,j} \log(p_{i,j})\quad s.t.\quad H(y) \geq \gamma
\end{align*}

Compared with the problem consisting of $K$ constraints in Eqn.~\ref{eq:obsize}, there is only a single constraint that controls the size of all clusters simultaneously.
According to the dual theory~\cite{boyd2004convex}, the problem is equivalent to
\begin{eqnarray}\label{eq:entropy}
\min_{y_i\in\Delta} -\sum_{i=1}^N\sum_{j=1}^K y_{i,j} \log(p_{i,j}) - \alpha H(y)
\end{eqnarray}

Since one-hot labels of all instances are kept in memory, the optimal solution can be obtained by enumerating. When optimizing the label for $x_i$ and fixing other labels, the closed-form solution for $y_i$ is
\begin{align}\label{eq:yentropy}
&y_{i,j}^t = \left\{\begin{array}{ll}1\quad&j=\arg\min_{j} -\log(p_{i,j})-\alpha H(y^{t-1},y_{i:j})   \\ 0\quad&o.w. \end{array} \right.
\end{align}
where $H(y^{t-1},y_{i:j})$ replaces the previous label $y_i$ of $\x_i$ by $j$ as $H(y^{t-1},y_{i:j})=H([y_1^{t-1},\dots,y_{i-1}^{t-1},y_{i+1}^{t-1},\dots,y_{n}^{t-1};y_{i,j}=1])$.
The convergence is guaranteed as in the following theorem.
\begin{thm}
Let $\mathcal{L}(Y)$ denote the objective in Eqn.~\ref{eq:entropy}. If updating cluster assignments sequentially according to Eqn.~\ref{eq:yentropy}, we have $\mathcal{L}(Y^t)\leq \mathcal{L}(Y^{t-1})$.
\end{thm}
\begin{proof}
It is directly from the optimality of the closed-form solution in Eqn.~\ref{eq:yentropy}.
\end{proof}

\paragraph{Remark} While some deep clustering methods apply the entropy regularization for learning, it is defined over a mini-batch of instances and optimized by SGD~\cite{GansbekeVGPG20}. On the contrary, our method leverages the entropy of the entire data set that can capture the global distribution better. Moreover, the entropy in our method can be maximized with a closed-form solution, which is more efficient for learning.

With the two-view optimization, the updating becomes
\begin{align*}
j=\arg\min_{j} -(\log(p_{i,j}^1)+\log(p_{i,j}^2))/2-\alpha H(y^{t-1},y_{i:j}) 
\end{align*}

\paragraph{Update of $\w$} With fixed $\x_i^t$ and the updated pseudo label $\{y^t\}$, cluster centers can be optimized by minimizing SeCu loss over two views of augmentations by SGD
\[\min_{\{\w_j\}} \sum_{i=1}^N \ell_{\mathrm{SeCu}}(\x_i^{1}, y_i^t)+\ell_{\mathrm{SeCu}}(\x_i^{2}, y_i^t)\]
Alternatively, $\w$ can be updated by a closed-form solution
\[\w_j^{t:B} = \Pi_{\|\w\|_2=1}(\frac{\sum_{i:y_i=j}^B (1-p_{i,j})\x_i^t}{\sum_{i:y_i=j}^B 1- p_{i,j}})
\]
where $B$ denotes the total number of instances received in the current epoch.

The proposed method is easy to implement and the whole framework is illustrated in Alg.~\ref{alg:secu}. Compared with the supervised discrimination, our main computational overhead is from cluster assignment, which is negligible due to the online strategy.

\begin{algorithm}[!ht]
\caption{Pseudo-code of \textbf{S}tabl\textbf{e} \textbf{C}l\textbf{u}ster Discrimination (SeCu) for One-stage Deep Clustering.}\label{alg:secu}
\label{alg:secu}
\definecolor{codeblue}{rgb}{0.25,0.5,0.5}
\lstset{
  backgroundcolor=\color{white},
  basicstyle=\fontsize{7.2pt}{7.2pt}\ttfamily\selectfont,
  columns=fullflexible,
  breaklines=true,
  captionpos=b,
  commentstyle=\fontsize{7.2pt}{7.2pt}\color{codeblue},
  keywordstyle=\fontsize{7.2pt}{7.2pt},
}
\begin{lstlisting}[language=python]
# f: encoder network
# w: cluster centers
# w_p: centers from the last epoch
# y: list of pseudo one-hot labels (Nx1)
# tau: weight for labels from last epoch
# lambda: temperature

# keep last cluster centers before each epoch
w_p = w.detach() 
# train one epoch
for x in loader:  # load a minibatch with b samples
    x_1, x_2 = f(aug(x)), f(aug(x)) # two random views
    y_x = y(x_id) # retrieve label from last epoch
    # compute prediction for each view
    p_1 = softmax(x_1 @ w_p / lambda)
    p_2 = softmax(x_2 @ w_p / lambda)
    # obtain soft label for discrimination
    y_1 = tau * y_x + (1-tau) * p_2 
    y_2 = tau * y_x + (1-tau) * p_1
    # loss for representation learning
    loss_x = (SeCu(p_1, y_1) + SeCu(p_2, y_2)) / 2
    # update prediction for clustering
    p_1 = softmax(x_1.detach() @ w / lambda)
    p_2 = softmax(x_2.detach() @ w / lambda)
    # update cluster assignments with constraints
    y(x_id) = y_x = cluster_assign(p_1, p_2) 
    # loss for cluster centers
    loss_w = (SeCu(p_1, y_x) + SeCu(p_2, y_x)) / 2
    # update encoder and cluster centers
    loss = loss_x + loss_w
    loss.backward() 
\end{lstlisting}
\end{algorithm}

\section{Experiments}
The performance of the proposed method is evaluated on CIFAR-10~\cite{krizhevsky2009learning} , CIFAR-100~\cite{krizhevsky2009learning}, STL-10\cite{CoatesNL11} and ImageNet~\cite{RussakovskyDSKS15}. To demonstrate the generalization performance of SeCu, we follow the setting in SCAN~\cite{GansbekeVGPG20} that trains and evaluates models on the standard train/test split provided by the original data set. For STL-10 that contains an additional noisy unlabeled data set, we first learn the model on the training and noisy set, and then on the target training data as suggested by other methods~\cite{GansbekeVGPG20,ZhongW0HDNL021}. The standard metrics for clustering, including clustering accuracy (ACC), normalized mutual information (NMI) and adjusted rand index (ARI) are reported for evaluation. For a fair comparison, we follow the common practice to configure the architecture and training protocol for our method. Concretely, ResNet-50~\cite{HeZRS16} is applied on ImageNet, while ResNet-18 is adopted for other data sets. The key settings for ResNet-18 are elaborated as follows, while those of ResNet-50 follow \cite{coke}. More details can be found in the appendix.

\paragraph{Architecture} Besides backbone network, a 2-layer MLP head is attached and the output dimension is $128$. The similar architecture is adopted in \cite{coke,ZhongW0HDNL021}. With the MLP head, the total number of parameters is almost the same as the original ResNet-18. The discussion about the MLP head can be found in Sec.~\ref{sec:mlp}. After the MLP head, the learned representation will be classified by a fully-connected (FC) layer that encodes cluster centers. The size of the FC layer is $128\times K$. As a clustering method, we let $K$ be the number of ground-truth classes as in \cite{GansbekeVGPG20} and have the direct prediction from the model as cluster assignments for evaluation. Following \cite{AsanoRV20a,GansbekeVGPG20}, $10$ different classification heads are applied as multi-clustering, which benefits the target clustering. To avoid the problem of selecting an appropriate head for evaluation when $10$ heads have the same $K$, we have a varying $cK$ as the number of clusters for each head, where $c\in\{1,\dots,10\}$. The prediction from the head with $c=1$ is reported for comparison.

\paragraph{Optimization} To reduce efforts of parameter tuning and reuse the parameter of \cite{coke}, we have $\x_i^\top\w_j$ in lie of $\log(p_{i,j})$ in Eqn.~\ref{eq:y} and \ref{eq:yentropy} for assignment. Searching the parameter for the latter one can achieve the similar performance. Before training the encoder network, one epoch is adopted to initialize cluster assignments and centers. The encoder network is optimized by SGD with a batch size of $128$. The momentum and learning rate are set to be $0.9$ and $0.2$, respectively. Moreover, the first 10 epochs are applied for warm-up and the cosine decay for learning rate is adopted subsequently. For CIFAR-10, the model is optimized by $400$ epochs and cluster centers are learned by SGD with a constant learning rate of $1.2$. For other challenging data sets, epochs and learning rate for centers are $800$ and $0.8$, respectively. $\tau$ for the soft label as in Eqn.~\ref{eq:soft} and the temperature $\lambda$ are fixed as $0.2$ and $0.05$. For the size constraint, the lower-bound parameter $\gamma$ and the learning rate of dual variables $\eta_\rho$ are set to be $0.9$ and $0.1$, respectively. The only parameter $\alpha$ for the proposed global entropy constraint is set to be $6N/50$ , where $N$ is the total number of instances. The weight is scaled according to the ablation study on CIFAR-10 as illustrated in Sec.~\ref{sec:alpha}. 

\paragraph{Augmentation} Augmentation is essential for the success of unsupervised representation learning~\cite{ChenK0H20}. For a fair comparison, we apply the prevalent settings in existing work~\cite{ChenK0H20,GrillSATRBDPGAP20}. Concretely, random crop, color jitter, random gray scale, Gaussian blur, solarize and random horizontal flip are included as augmentations. Considering that the resolution of images in CIFAR and STL-10 is much smaller than those in ImageNet, we let the minimal scale of random crop be $0.3$, $0.2$, and $0.2$ for CIFAR-10, CIFAR-100, and STL-10, respectively. That for ImageNet is kept as $0.08$. Other parameters are the same for all data sets and many recent work share the similar augmentation~\cite{abs-2104-02057,GrillSATRBDPGAP20,coke,ZhongW0HDNL021}. 

All experiments on small data sets are implemented on a single V100, while $8$ GPUs are sufficient for ImageNet.

\subsection{Ablation Study}
To illustrate the behavior of the proposed method, we study the effect of parameters in SeCu on CIFAR-10 with the size constraint.

\subsubsection{Effect of MLP Head}\label{sec:mlp}
The MLP head shows better generalization than the FC layer in representation learning~\cite{ChenK0H20}, but its application in deep clustering has not been evaluated systematically. We compare different architectures of MLP head and summarize results in Table~\ref{ta:mlp}. Besides the standard MLP layer for projection, there can be another MLP layer attached to the projection MLP head, which is referred as the prediction head~\cite{GrillSATRBDPGAP20}.

\begin{table}[!ht]
\centering
\begin{tabular}{|l|l|l|l|l|l|}\hline
\#Proj&\#Pred&\#Para&ACC&NMI&ARI\\\hline
0&0&11.45M&85.5&76.2&73.1\\
1&0&11.30M&87.6&78.7&76.7\\
2&0&11.57M&88.1&79.4&77.6\\
3&0&11.83M&88.1&79.4&77.6\\
3&2&11.97M&88.2&79.6&77.8\\\hline
\end{tabular}
\caption{Comparison of MLP head on CIFAR-10. ``\#Proj'' and ``\#Pred'' indicate the number of layers in MLP for projection head and prediction head, respectively. ACC (\%), NMI (\%) and ARI (\%) are reported. Baseline ACC of SCAN~\cite{GansbekeVGPG20} is $81.8\%$.}\label{ta:mlp}
\end{table}

First, with the standard architecture denoted by 0 for projection and prediction, SeCu is already $3.7\%$ better than SCAN on ACC and it confirms that the one-stage deep clustering strategy can learn more consistent features than two-stage methods. If including a single layer for SeCu, the performance can be improved by $2.1\%$, which shows that adding a layer with the reduced dimension in head is helpful for deep clustering. Moreover, due to the dimension reduction from the projection layer, the size of cluster centers decreases and the total number of parameters in the model is even smaller than that of the original network. When increasing the number of layers and having a 2-layer MLP head, ACC can be further improved by $0.5\%$ and the model size is only slightly increased. After that, there is no significant gain by applying more complicated MLP head. It is because that the 2-layer MLP head is sufficient for small data sets and we will fix 2-layer MLP as the head for the rest experiments except ImageNet, where 3-layer projection and 2-layer prediction are applied as in \cite{coke}. 

\subsubsection{Effect of Negative Instances}
After the ablation experiments for the architecture, we study the effect of the proposed stable cluster discrimination loss in Eqn.~\ref{eq:pce}. Table~\ref{ta:pce} compares it to the standard cross entropy loss (CE) that keeps the gradient from negative instances.

\begin{table}[!ht]
\centering
\begin{tabular}{|l|c|c|c|c|c|}\hline
Loss&\#Max&\#Min&ACC&NMI&ARI\\\hline
CE&6,486&4,501&51.6&50.2&35.0\\
SeCu&5,190&4,556&88.1&79.4&77.6\\\hline
\end{tabular}
\caption{Comparison of loss functions on CIFAR-10. ``\#Max'' and ``\#Min'' indicate the maximal and minimal size of obtained clusters.}\label{ta:pce}
\end{table}

First, we can observe that the proposed SeCu loss outperforms the standard cross entropy loss by a large margin of $36.5\%$ on ACC, which confirms the benefit of stable training. Then, by investigating the distribution of learned clusters, we find that the minimal size of clusters obtained from different loss functions is almost the same due to the lower-bound constraint, while the maximal size varies. The largest cluster obtained by SeCu contains 5,190 instances, which is close to the ground-truth distribution of CIFAR-10. On the contrary, the dominating cluster found by CE loss has 6,486 instances that is significantly different from the uniform distribution in CIFAR-10. It is because that the negative instances in CE loss can perturb the learning of centers and result in sub-optimal cluster assignments.

\subsubsection{Effect of Updating Criterion for Cluster Centers}
The cluster centers in SeCu can be optimized by SGD or updated by a closed-form solution as in Eqn.~\ref{eq:w}. We compare these two strategies and summarize the performance in Table~\ref{ta:cf10opt}.

\begin{table}[!ht]
\centering
\begin{tabular}{|l|c|c|c|}\hline
Methods&ACC&NMI&ARI\\\hline
SeCu-CF &87.4$\pm$0.5&78.5$\pm$0.6&76.4$\pm$0.9\\
SeCu-SGD &88.1$\pm$0.3&79.3$\pm$0.4&77.5$\pm$0.4\\\hline
\end{tabular}
\caption{Comparison of optimization for cluster centers on CIFAR-10. Mean$\pm$std over 8 trails are reported.}\label{ta:cf10opt}
\end{table}

Let ``SeCu-CF'' and ``SeCu-SGD'' be the closed-form update and the SGD optimization, respectively. We can observe that SeCu-CF has the similar performance to SeCu-SGD, which shows the effectiveness of the closed-form updating. Moreover, SeCu-SGD has a smaller standard deviation than SeCu-CF due to the smoothness from the momentum in SGD. Compared with SeCu-CF, SeCu-SGD has an additional parameter of the learning rate for centers. Hence, SeCu-CF can be applied to evaluate the preliminary performance of SeCu with less tuning efforts.

\subsubsection{Effect of $\alpha$ in Entropy Constraint}\label{sec:alpha}

Besides the size constraint, the proposed entropy constraint is also effective for balancing cluster assignments. Table~\ref{ta:alpha} demonstrates the results with different $\alpha$. 

\begin{table}[!ht]
\centering
\begin{tabular}{|l|l|l|l|l|l|}\hline
$\alpha$&\#Max&\#Min&ACC&NMI&ARI\\\hline
20,000&5,056&4,913&85.3&76.3&73.1\\
8,000&5,115&4,658&87.9&79.0&77.1\\
6,000&5,120&4,579&88.0&79.4&77.4\\
4,000&5,376&4,324&84.9&75.4&72.2\\
100&28,575&0&19.9&38.7&17.0\\
0&50,000&0&10.0&0.0&0.0\\\hline
\end{tabular}
\caption{Comparison of $\alpha$ in Eqn.~\ref{eq:entropy} on CIFAR-10. }\label{ta:alpha}
\end{table}

First, a large $\alpha$ will regularize the distribution to be uniform and result in a sub-optimal performance. By decreasing the weight, the assignment becomes more flexible and a desired performance can be observed when $\alpha=6,000$. However, a small $\alpha$ will lead to an imbalanced distribution, which is inappropriate for the balanced data sets such as CIFAR and STL. $\alpha=0$ discards the constraint and incurs collapsing. Evidently, the entropy constraint can balance the size of clusters effectively

Since the entropy is defined on the whole data set and CIFAR-10 contains $50,000$ instances for training, $\alpha$ will be scaled as $6N/50$ for different data sets, where $N$ denotes the total number of instances in the training set. More ablation experiments can be found in the appendix.

\subsection{Comparison with State-of-the-Art}
After the ablation study, experiments are conducted on benchmark data sets to compare with state-of-the-art methods. The results averaged over 8 trials and the best performance among them are reported for our method. SeCu with size constraint and entropy constraint are denoted as ``SeCu-Size'' and ``SeCu-Entropy'', respectively. Table~\ref{ta:cluster} shows the performance of different methods, where two-stage methods with a pre-training phase are marked for ``Two-stage''.

\begin{table*}[!ht]
\centering
\begin{tabular}{|l|c|c|c|c|c|c|c|c|c|c|}\hline
\multirow{2}*{Methods}&\multirow{2}*{Two-stage}&\multicolumn{3}{c|}{CIFAR10}&\multicolumn{3}{c|}{CIFAR100-20}&\multicolumn{3}{c|}{STL10}\\\cline{3-11}
&&ACC&NMI&ARI&ACC&NMI&ARI&ACC&NMI&ARI\\\hline
Supervised&&93.8&86.2&87.0&80.0&68.0&63.2&80.6&65.9&63.1\\\hline
DeepCluster~\cite{CaronBJD18}&&37.4&-&-&18.9&-&-&33.4&-&-\\\hline
IIC~\cite{JiVH19}&&61.7&51.1&41.1&25.7&22.5&11.7&59.6&49.6&39.7\\\hline
PICA~\cite{HuangGZ20} (best)&&69.6&59.1&51.2&33.7&31.0&17.1&71.3&61.1&53.1\\\hline
Pretext~\cite{ChenK0H20}+k-means&\checkmark&65.9&59.8&50.9&39.5&40.2&23.9&65.8&60.4&50.6\\\hline
SCAN~\cite{GansbekeVGPG20} (mean) &\checkmark&81.8 & 71.2 & 66.5& 42.2 & 44.1 & 26.7& 75.5&65.4&59.0\\\hline
NNM~\cite{DangD0WH21}&\checkmark&84.3&74.8&70.9&47.7&48.4&31.6&80.8&69.4&65.0\\\hline
GCC~\cite{ZhongW0HDNL021}&\checkmark&85.6&76.4&72.8&47.2&47.2&30.5&78.8&68.4&63.1\\\hline
CoKe~\cite{coke} (mean) &&85.7&76.6&73.2&49.7&49.1&33.5&-&-&-\\\hline
SeCu-Size (mean) &&88.1&79.3&77.5&50.0&50.7&35.0&80.2&69.4&63.9\\
SeCu-Size (best)&&\textbf{88.5}&\textbf{79.9}&\textbf{78.2}&\textbf{51.6}&\textbf{51.6}&\textbf{36.0}&\textbf{81.4}&\textbf{70.7}&\textbf{65.7}\\\hline
SeCu-Entropy (mean)&&88.0&79.2&77.2&49.9&50.6&34.1&79.5 &68.7&63.0\\
SeCu-Entropy (best)&&88.2&79.7&77.7&51.2&51.4&34.9&80.5&69.9&64.4\\\hline
\multicolumn{11}{|l|}{\textit{with self-labeling:}}\\\hline
SCAN~\cite{GansbekeVGPG20}&\checkmark&88.3&79.7&77.2&50.7&48.6&33.3&80.9&69.8&64.6 \\
GCC~\cite{ZhongW0HDNL021}&\checkmark&90.1&-&-&52.3&-&-&83.3&-&- \\
SeCu &&\textbf{93.0}&\textbf{86.1}&\textbf{85.7}&\textbf{55.2}&\textbf{55.1}&\textbf{39.7}& \textbf{83.6} & \textbf{73.3}&\textbf{69.3}\\\hline
\end{tabular}
\caption{Comparison of clustering methods on benchmark data sets. ``Two-stage'' denotes an additional pre-training stage.}\label{ta:cluster}
\end{table*}

First, as a one-stage method, SeCu outperforms two-stage methods by a margin of about $3\%$ over all metrics on CIFAR-10 and CIFAR-100-20. It demonstrates that learning representations and clustering simultaneously is essential for deep clustering. The two-stage method NNM\cite{DangD0WH21} shows the comparable performance to SeCu on STL-10. It is because that STL-10 only constrains $5,000$ target instances for training and both of NNM and SeCu already achieve the accuracy of supervised learning. Second, compared with the one-stage method without sufficient representation learning, i.e., PICA, SeCu demonstrates a better performance by about $20\%$ on CIFAR-10 and CIFAR-100-20, and $10\%$ on STL-10, which shows the importance of optimizing representations for instances. Third, with the same size constraint, SeCu surpasses CoKe by $4.3\%$ for ARI on CIFAR-10, which confirms the effectiveness of the proposed learning objective. Finally, compared with size constraint in SeCu-Size, SeCu-Entropy demonstrates a competitive performance over all data sets. Since the entropy constraint is optimized by enumeration without introducing any auxiliary variable, it can be an substitute to the size constraint for cluster assignments.

After the clustering phase, self-labeling~\cite{GansbekeVGPG20} is an effective strategy to further improve the performance. We try it for SeCu-Size with the setting in \cite{GansbekeVGPG20} and report the result in Table~\ref{ta:cluster}. Interestingly, self-labeling is also effective for SeCu and helps approach the supervised accuracy on CIFAR-10. As a one-stage method, SeCu has the strong augmentation for representation learning, which may introduce additional noise for clustering. Self-labeling has the weak augmentation for fine-tuning and refines clustering. 

\begin{figure}[!ht]
\centering
\includegraphics[height = 1.5in]{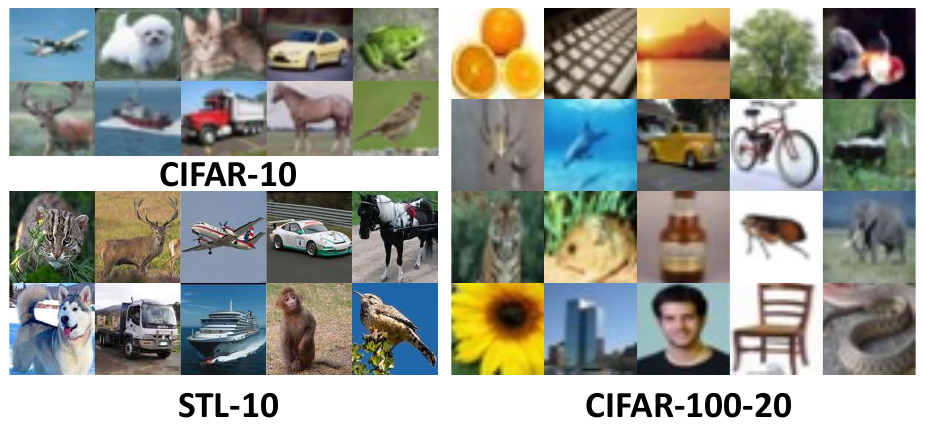}
\caption{Illustration of exemplar images.}\label{fig:exemplar}
\end{figure}

\paragraph{Exemplar Images}
Finally, we show exemplar images from clusters obtained by SeCu for different data sets. Fig.~\ref{fig:exemplar} illustrates the images that are close to cluster centers in CIFAR-10, CIFAR-100-20, and STL-10, respectively. Evidently, the proposed method can obtain the ground-truth class structure for CIFAR-10 and STL-10. Moreover, with a large intra-class variance, SeCu can still observe appropriate clusters for CIFAR-100-20.

Considering that labels of super-classes in CIFAR-100-20 contain vague concepts such as ``vehicles 1'' and ``vehicles 2'', we include results on the target 100 classes in Table~\ref{ta:cf100}, which may benefit future research. The experiment settings for CIFAR-100 are identical to that for CIFAR-100-20, except the learning rate of centers, which is reduced to $0.01$.

\begin{table}[!ht]
\centering
\begin{tabular}{|l|c|c|c|}\hline
Methods&ACC&NMI&ARI\\\hline
SeCu-Size (mean)&46.4&61.7&32.3 \\
SeCu-Size (best)&47.1&61.9&32.6 \\
SeCu-Entropy (mean) & 45.3&60.7&31.5\\
SeCu-Entropy (best) & 46.7&61.0&32.2\\
SeCu$^\dagger$&51.3&65.2&37.1\\\hline
\end{tabular}
\caption{Results of SeCu on CIFAR-100. $^\dagger$ denotes a method with self-labeling.}\label{ta:cf100}
\end{table}

\subsection{Comparison on ImageNet}
Now we evaluate SeCu on the challenging ImageNet data set in Table~\ref{ta:imagenet}, where the cost of many two-stage methods becomes intractable and only available baselines are included. First, both variants of SeCu can achieve $51\%$ ACC on ImageNet, which is better than SCAN with self-labeling by a clear margin of $11\%$. With the same $10$ clustering heads and size constraint, SeCu-Size is still $3.8\%$ better than CoKe on ACC. Compared with the objective in CoKe, SeCu is tailored for deep clustering and the comparison confirms the efficacy of our method. With self-labeling, SeCu shows state-of-the-art performance of $53.5\%$ ACC on ImageNet. It demonstrates that our method is applicable for large-scale data sets. 

\begin{table}[!ht]
\centering
\begin{tabular}{|l|c|c|c|}\hline
Methods&ACC&NMI&ARI\\\hline
SCAN$^\dagger$&39.9&72.0&27.5\\
CoKe$^\ddagger$&47.6&76.2&35.6\\
SeCu-Size&51.4&78.0&39.7\\
SeCu-Entropy&51.1&77.5&39.1\\
SeCu$^\dagger$&53.5&79.4&41.9\\\hline
\end{tabular}
\caption{Comparison on ImageNet. $^\dagger$ denotes a method with self-labeling. $^\ddagger$ is our reproduction with $10$ clustering heads.}\label{ta:imagenet}
\end{table}

\section{Conclusion}
In this work, a novel stable cluster discrimination task that considers the hardness of instances is proposed for one-stage deep clustering. To avoid collapsing in representations, a global entropy constraint with theoretical guarantee is investigated. Finally, the comparison with state-of-the-art methods confirms the effectiveness of our proposal. While recent work shows that leveraging nearest neighbors improves deep clustering~\cite{DangD0WH21, GansbekeVGPG20,ZhongW0HDNL021}, incorporating such information in SeCu can be our future work.

{\small
\bibliographystyle{ieee_fullname}
\bibliography{secu}
}
\appendix
\section{Theoretical Analysis}
\subsection{Proof of Proposition~1}
\begin{proof}
Suppose for contradiction that there are $K-b-1$ clusters without any positive instances. Then, $b+1$ clusters have positive instances. Since a positive instance cannot be shared by different clusters, the total number of instances is no less than $b+1$, which contradicts the batch size of $b$.
\end{proof}
\subsection{Proof of Proposition~2}
\begin{proof}
Assuming that each cluster has the same number of instances and $\mu_i = E[\x_i]$, we have
{\small
\begin{align*}
&\mathrm{Var}_{pos} = E_{\x_i}[\|\x_i - \mu_i\|_2^2] = 1-\|\mu_i\|_2^2 = 1-a^2\\
&\mathrm{Var}_{neg} = E_{\x_j}[\|\x_j - \frac{1}{K-1}\sum_j^{K-1}\mu_j\|_2^2] = 1-\|\frac{1}{K-1}\sum_j^{K-1}\mu_j\|_2^2
\end{align*}}
If assuming a uniform distribution of centers such that $E_\mu[\mu]=\mathbf{0}$, we have $\|\frac{1}{K-1}\sum_j^{K-1}\mu_j\|_2^2 = \frac{a^2}{K-1}$ and
$\mathrm{Var}_{neg} = 1-a^2/(K-1)$. Therefore 
$\mathrm{Var}_{neg} = (\frac{K-2}{(K-1)(1-a^2)} + \frac{1}{K-1}) \mathrm{Var}_{pos}$.
\end{proof}

\subsection{Proof of Theorem~1}
\begin{proof}
When fixing $\x_i$ and $\{y_i\}$, the optimization problem for centers can be written as
\begin{align*}
\min_{\{\w_j\}} &\sum_i \log(\exp(\x_i^\top \w_{y_i}/\lambda)+\sum\limits_{k:k\neq y_i}\exp(\x_i^\top\tilde{\w}_k/\lambda))\\
&- \x_i^\top \w_{y_i}/\lambda
\end{align*}
Since $\x_i$ and $\w_j$ have the unit length, the problem is equivalent to
\begin{align}\label{eq:objw}
\min_{\{\w_j\}} &\sum_i \log(\exp(-\|\x_i- \w_{y_i}\|_2^2/2\lambda)\nonumber\\
&+\sum\limits_{k:k\neq y_i}\exp(-\|\x_i - \tilde{\w}_k\|_2^2/2\lambda)) + \|\x_i -  \w_{y_i}\|_2^2/2\lambda
\end{align}
We can obtain the solution by letting the gradient of $\w$ be 0. Nevertheless, we will introduce an alternating method for better demonstration.

By introducing an auxiliary variable $q_i$ as the distribution over centers, the problem can be further written as
\begin{align}\label{eq:obj}
&\min_{\{\w_j\}} \sum_i \max_{q_i\in\Delta'}-q_{i,y_i}\|\x_i- \w_{y_i}\|_2^2/2\lambda\\
&+\sum\limits_{k:k\neq y_i}-q_{i,k}\|\x_i - \tilde{\w}_k\|_2^2/2\lambda + H(q_i) + \|\x_i -  \w_{y_i}\|_2^2/2\lambda\nonumber
\end{align}
where $H(q_i)=-\sum_j q_{i,j}\log(q_{i,j})$ measures the entropy of the distribution and $\Delta' = \{q_i|\sum_{j=1}^K q_{i,j}=1, \forall j, q_{i,j}\geq 0\}$. 
We note that $q_i$ has the closed-form solution according to the K.K.T. condition~\cite{boyd2004convex} as
\begin{eqnarray}\label{eq:q}
q_{i,j} = \frac{\exp(\x_i^\top \w_j/\lambda)}{\sum_k^K \exp(\x_i^\top \w_k/\lambda)}=p_{i,j}
\end{eqnarray}
Taking it back to the problem and letting the gradient for centers be 0, the optimal solution $\w^*$ should satisfy the property
\[\w_j =\frac{\sum_{i:y_i=j} (1-p_{i,j})\x_i}{\sum_{i:y_i=j} 1- p_{i,j}}\]
With the unit length constraint and K.K.T. condition~\cite{boyd2004convex}, it will be projected as
\begin{eqnarray}\label{eq:w}
\w_j = \Pi_{\|\w\|_2=1}(\frac{\sum_{i:y_i=j} (1-p_{i,j})\x_i}{\sum_{i:y_i=j} 1- p_{i,j}})
\end{eqnarray}

Now, we demonstrate the effect of the closed-form solution. Let $\LL(\w)$ denote the objective in Eqn.~\ref{eq:objw} and we have
\[\nabla\LL(\w) = \w - \frac{\sum_{i:y_i=j} (1-p_{i,j})\x_i}{\sum_{i:y_i=j} 1- p_{i,j}}\]
According to gradient descent (GD), centers can be updated as
\[\w^t = \Pi_{\|\w\|_2=1}(\w^{t-1} - \eta_w\nabla\LL(\w^{t-1}))\]
The target solution can be obtained by setting $\eta_w=1$. Therefore, the closed-form solution can be considered as the vanilla gradient descent with the constant learning rate of $1$, which suggests a constant learning rate for cluster centers.
\end{proof}

\section{SeCu with Upper-bound Size Constraint}
We introduce the upper-bound size constraint for the completeness, while the lower-bound constraint is sufficient in our experiments. With the additional upper-bound size constraint, the objective for SeCu becomes
\begin{eqnarray*}
&&\min_{\theta_f, \{\w_j\}, y_i\in\Delta} \sum_{i=1}^N\sum_{j=1}^K \ell_{\mathrm{SeCu}}(x_i,y_i)\nonumber\\
s.t. &&  \sum_i y_{i,j} \geq \gamma N/K,\quad j=1,\dots, K\nonumber\\
&& \sum_i y_{i,j} \leq \gamma' N/K,\quad j=1,\dots, K
\end{eqnarray*}

Compared with the variant containing the lower-bound constraint, the difference is from the updating for cluster assignments.

When fixing $\x_i$ and cluster centers $\{\w_j\}$, cluster assignments will be updated by solving an assignment problem as
\begin{eqnarray*}
&&\min_{y_i\in\Delta} \sum_{i=1}^N\sum_{j=1}^K -y_{i,j}\log(p_{i,j})\nonumber\\
s.t. && \sum_i y_{i,j} \geq \gamma N/K,\quad j=1,\dots, K\nonumber\\
&& \sum_i y_{i,j} \leq \gamma' N/K,\quad j=1,\dots, K
\end{eqnarray*}

We extend the dual-based method in \cite{coke} to update labels in an online manner. Let $\rho_j$ and $\rho'_j$ denote dual variables for the $j$-th lower-bound and upper-bound constraints, respectively. When a mini-batch of $b$ examples arrive at the $r$-th iteration of the $t$-th epoch, the cluster assignments for instances in the mini-batch can be obtained via a closed-form solution as
\begin{eqnarray*}
y_{i,j}^t = \left\{\begin{array}{ll}1\quad&j=\arg\min_j -\log(p_{i,j})-\rho_j^{r-1}+{\rho'}_j^{r-1}  \\ 0\quad&o.w. \end{array} \right.
\end{eqnarray*}
After that, the dual variables will be updated as
\begin{eqnarray*}
&&\rho_j^r = \max(0,\rho_j^{r-1} - \eta_\rho \frac{1}{b}\sum_{s=1}^b( y_{s,j}^t-\gamma/K))\nonumber\\
&&{\rho'}_j^r = \max(0,{\rho'}_j^{r-1} + \eta_\rho \frac{1}{b}\sum_{s=1}^b( y_{s,j}^t-\gamma'/K))
\end{eqnarray*}
where $\eta_\rho$ is the learning rate of dual variables. Without dual variables, the online assignment is degenerated to a greedy strategy. Intuitively, dual variables keep the information of past assignments and help adjust the current assignment adaptively to satisfy the global constraint. 

\section{Experiments}

\subsection{More Implementation Details} 

\paragraph{Experiments on STL-10} Unlike CIFAR, STL-10 has an additional noisy data set for unsupervised learning. Therefore, the temperature for optimizing cluster centers is increased to $1$ to learn from the noisy data, while that for representation learning remains the same. Moreover, the weight of the entropy constraint is increased to $26,460$ for the first stage training. It is reduced to $600$ in the second stage according to the proposed scaling rule, when only clean training set is used. Finally, for the second stage, only the target clustering head is kept for training and the learning rate for the encoder network is reduced from $0.2$ to $0.002$ for fine-tuning. Other parameters except the number of epochs are the same as the first stage. The number of training epochs for the first and the second stage is $800$ and $100$, respectively.

\paragraph{Experiments on ImageNet} We reuse the settings in \cite{coke} for our method while searching the optimal parameters may further improve the performance. Concretely, the model is optimized by LARS~\cite{abs-1708-03888} with $1,000$ epochs, where the weight decay is $10^{-6}$, the momentum is $0.9$ and the batch size is $1,024$. The learning rate for the encoder network is $1.6$ with the cosine decay and 10-epoch warm-up. The ratio in the lower-bound size constraint and the learning rate of dual variables are set to be $0.4$ and $20$, respectively. The learning rate for cluster centers is fixed as $4.2$ in SeCu-Size. For the entropy constraint, $\alpha$ is $90,000$ and the learning rate for cluster centers is gradually increased from $0$ to $5.6$ according to the negative cosine function in $[0,\pi]$. 

\paragraph{Self-labeling}
Self-labeling is to fine-tune the model by optimizing the strong augmentation with pseudo labels from the weak augmentation, where the strong augmentation here is still much milder than that for pre-training. For a fair comparison, the same weak and strong augmentations as in \cite{GansbekeVGPG20} are applied for SeCu. Besides, SGD is adopted for self-labeling with $100$ epochs on small data sets and $11$ epochs on ImageNet. The batch size is $1,024$ and momentum is $0.9$, which are the same as \cite{GansbekeVGPG20}. Before selecting the confident instances by the prediction from the weak augmentation with a threshold of 0.9, we have a warm-up period with $10$ epochs, where all instances are trained with the fixed pseudo label from the assignment of pre-trained SeCu. 

\subsection{Ablation Study}

\subsubsection{Effect of Output Dimension}

Given the 2-layer MLP head, we investigate the effect of the output dimension by varying the value in $\{64,128,256,512\}$. Table~\ref{ta:dim} shows the performance of different dimensions.

\begin{table}[!ht]
\centering
\begin{tabular}{|l|c|c|c|}\hline
Output Dim&ACC&NMI&ARI\\\hline
64&88.0&79.3&77.4\\
128&88.1&79.4&77.6\\
256&88.2&79.3&77.5\\
512&87.8&79.0&77.2\\\hline
\end{tabular}
\caption{Comparison of the output dimension by the MLP head. }\label{ta:dim}
\end{table}

We can observe that the performance is quite stable with a small number of features. It is because that a low-dimensional space can capture the similarity with the standard distance metric better than a high-dimensional space. We will keep the output dimension as $128$, which is the same as the existing work~\cite{ZhongW0HDNL021}.

\subsubsection{Effect of $\gamma$ in Size Constraint}
Now we study the effect of the size constraint in SeCu and Table~\ref{ta:gamma} shows the performance with different lower-bound ratio $\gamma$. 

\begin{table}[!ht]
\centering
\begin{tabular}{|l|l|l|l|l|l|}\hline
$\gamma$&\#Max&\#Min&ACC&NMI&ARI\\\hline
1&5,015&4,973&85.4&76.2&73.0\\
0.9&5,190&4,556&88.1&79.4&77.6\\
0.8&5,323&4,422&87.6&78.7&76.6\\
0.7&5,721&3,789&86.6&77.7&75.1\\\hline
\end{tabular}
\caption{Comparison of $\gamma$ for SeCu-Size on CIFAR-10. }\label{ta:gamma}
\end{table}

The same phenomenon as the entropy constraint can be observed. When $\gamma=1$, it implies a well-balanced clustering that each cluster contains the similar number of instances. Although the constraint can be satisfied with the dual-based updating, the performance degenerates due to the strong regularization for a balanced cluster assignment. By reducing $\gamma$ to $0.9$, the assignment is more flexible, which leads to a better pseudo label for representation learning. The assignment becomes more imbalanced if further decreasing $\gamma$. Therefore, we fix $\gamma=0.9$ for small data sets.

\subsubsection{Effect of Batch Size}

SeCu inherits the property of supervised discrimination that is insensitive to the batch size. We vary it in $\{32,64,128,256\}$ and show the ACC of SeCu-Size on CIFAR-10 in Table~\ref{ta:bsize}, which confirms its efficacy.

\begin{table}[!ht]
\centering
\begin{tabular}{|l|c|c|c|c|}\hline
Batch Size&32&64&128&256\\\hline
ACC(\%)&87.9&88.3&88.1&87.9\\\hline
\end{tabular}
\caption{Comparison of batch size for SeCu-Size on CIFAR-10. }\label{ta:bsize}
\end{table}

\end{document}